\documentclass[conference]{IEEEtran}
\usepackage{times}
\usepackage{amsthm, amsfonts, amsmath, amssymb, mathtools}
\usepackage[ruled, linesnumbered]{algorithm2e}
\SetKwInput{KwInput}{Input}
\SetKwInput{KwOutput}{Output}
\usepackage{xcolor}
\usepackage[numbers]{natbib}
\usepackage{multicol}
\usepackage[bookmarks=true]{hyperref}
\usepackage{booktabs}
\usepackage{multirow}
\usepackage[table]{xcolor}
\usepackage{adjustbox}
\usepackage{cancel}
\usepackage{bbm}
\usepackage{cuted}
\usepackage{capt-of} % Required for \captionof
\usepackage{tabularx}

\definecolor{lightgray}{gray}{0.9}
\newcommand{\shade}{\cellcolor{lightgray}}

\usepackage{pgfplots}
\pgfplotsset{compat=1.18}
\definecolor{color1}{RGB}{27, 158, 119}  % Teal (Ours Ph1)
\definecolor{color2}{RGB}{217, 95, 2}    % Orange (Ours Ph2)
\definecolor{color3}{RGB}{117, 112, 179} % Purple (VI)
\definecolor{color4}{RGB}{231, 41, 138}  % Pink (MC)
\definecolor{color5}{RGB}{102, 166, 30}  % Green (MC Dist)

\newcommand{\policy}{\pi}
\newcommand{\vfunc}{V}
\newcommand{\vfuncpi}{\vfunc_{\policy}}
\newcommand{\task}{\mathcal{I}}

\newcommand{\dyn}{f}
\newcommand{\tset}{\mathcal{L}}
\newcommand{\tfunc}{l}

\newcommand{\mdp}{\mathcal{M}}
\newcommand{\dfac}{\gamma}
\newcommand{\sspace}{\mathcal{S}}
\newcommand{\aspace}{\mathcal{A}}
\newcommand{\tdyn}{P}
\newcommand{\state}{s}
\newcommand{\action}{a}

\newcommand{\statetrajspace}{\Xi_\state^{\policy, \tdyn}}
\newcommand{\tidx}{k}
\newcommand{\vfuncpiover}{\tilde{\vfunc}_{\pi}}
\newcommand{\vfuncpiwd}{\overline{\vfunc}_{\pi}}
\newcommand{\dataset}{\mathcal{D}}
\newcommand{\tidxgoal}{N}

\newcommand{\reals}{\mathbb{R}}

\newcommand{\expectation}{\mathbb{E}}

\newtheorem{proposition}{Proposition}
\newtheorem{remark}{Remark}

\begin{document}

% paper title
\title{Offline Policy Evaluation for Manipulation Policies via Discounted Liveness Formulation}

\author{
Hao Wang$^{\dag, \ddag}$~~~
Joshua Bowden$^{\ddag}$~~~
Colton Crosby $^{\ddag}$~~~
Somil Bansal $^{\ddag}$~~~
\\
$^\dag$University of Southern California~~~
$^\ddag$Stanford University~~~
\\
\{\texttt{haowwang, joshuabowden, cjcrosby, somil}\}@stanford.edu
}

\maketitle

\begin{strip}
    \centering
    \includegraphics[width=\textwidth]{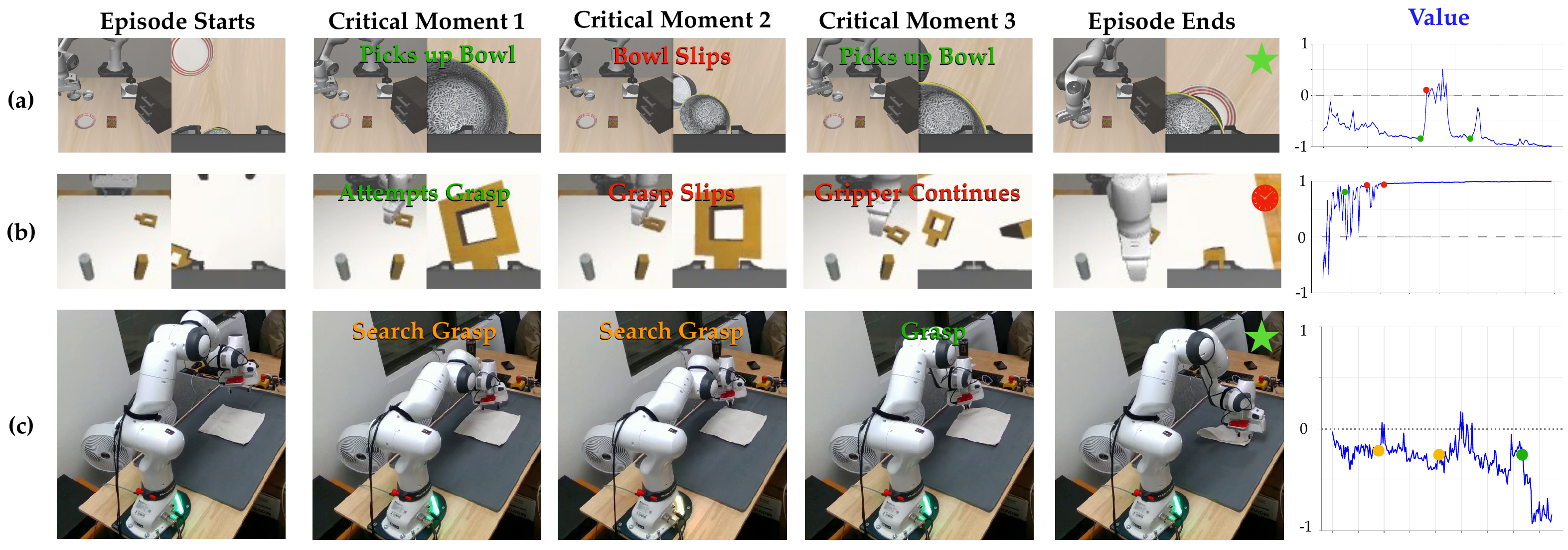}
    \captionof{figure}{In this work, we address the problem of offline policy evaluation for manipulation policies under sparse, episode-level rewards, where finite-horizon rollouts and non-monotonic task progression make standard value estimation methods unreliable. Our method formulates policy evaluation as a discounted liveness problem and learns a value function that maps states in the episode to a proxy of task progression. States close to goal completion are assigned low values although the value function is obtained only using sparse, episode-level rewards. We present episodes from three case studies alongside their corresponding value predictions. The base and wrist camera images are displayed side-by-side in the first five columns, while the final column visualizes the value predictions throughout the episode. We label and explain critical moments within the tasks and their associated values. As shown, the value function trends downward to $-1$ as the task progresses, increases when critical setbacks occur, and decreases again when the policy recovers.}
    \label{fig:teaser}
\end{strip}

\begin{abstract}
Policy evaluation is a fundamental component of the development and deployment pipeline for robotic policies. In modern manipulation systems, this problem is particularly challenging: rewards are often sparse, task progression of evaluation rollouts are often non-monotonic as the policies exhibit recovery behaviors, and evaluation rollouts are necessarily of finite length. This finite length introduces truncation bias, breaking the infinite-horizon assumptions underlying standard methods relying on Bellman equations/principle of optimality. In this work, we propose a framework for offline policy evaluation from sparse rewards based on a liveness-based Bellman operator. Our formulation interprets policy evaluation as a task-completion problem and yields a conservative fixed-point value function that is robust to finite-horizon truncation. We analyze the theoretical properties of the proposed operator, including contraction guarantees, and show how it encodes task progression while mitigating truncation bias. We evaluate our method on two simulated manipulation tasks using both a Vision–Language–Action model and a diffusion policy, and a cloth folding task using human demonstrations. Empirical results demonstrate that our approach more accurately reflects task progress and substantially reduces truncation bias, outperforming classical baselines such as TD(0) and Monte Carlo policy evaluation. The full implementation and experiments are available on the project website \footnote{\url{https://github.com/haowwang/offline_policy_eval_manipulation}}.

\end{abstract}

\IEEEpeerreviewmaketitle

\section{Introduction}
Recent years have seen rapid progress in autonomous robotic systems, driven in large part by generalist policies powered by high-capacity models and large-scale datasets \cite{rt_1, rt_2, pi_05, gemini_robotics_1_5, gr00t, transformer, droid_dataset, bridge_v2_dataset}. As these policies become increasingly capable and are deployed across a growing range of real-world manipulation tasks, reliable policy evaluation has emerged as a critical bottleneck. Policy evaluation plays a central role throughout the development lifecycle: it is used to iteratively improve policies \cite{pi_star_06}, assess readiness for deployment \cite{high_confidence_ope}, and guide policy steering and refinement \cite{steer_your_generalist}.

Fundamentally, the problem of policy evaluation centers on the question of \textit{how good is a given policy in completing its task?} While a large number of frameworks and procedures are developed to study this problem, Reinforcement Learning (RL) \cite{rl_book} remains one of the most enduringly popular frameworks for the problem, in which a \emph{value function} is computed to evaluate the performance of the target policy. The value function takes in a state and outputs a scalar that quantifies the expected performance of the policy from the starting state. Fittingly, the policy evaluation problem is a cornerstone of the RL theory, and classical methods such as iterative policy evaluation \cite{howard_dp_book}, Monte Carlo \cite{Rubinstein_mc_book, singh_replacing_eligibility_traces}, and temporal-difference learning such as TD(0) \cite{sutton_td_learning_1988}, have been applied to a wide array of problems. 

Despite its importance, offline policy evaluation for manipulation policies remains fundamentally challenging.
The difficulty stems from three properties that are ubiquitous in real-world manipulation. 
First, rewards are typically sparse and task-driven: beyond a binary signal indicating task completion, dense rewards are often unavailable or unreliable. Second, practical policy evaluation is necessarily time-limited—episodes terminate after a finite horizon due to manual time-outs. As a result, policy failures and episode truncations are indistinguishable in offline data, introducing the well-known \emph{truncation bias} in the value function, leading to systematic underestimation of policy performance and poor value estimates. Third, real-world manipulation policies often exhibit recovery behaviors, where they keep trying until task completion or manual time-out, and the resulting rollouts have non-monotonic task progression, complicating value function estimation. 

In this work, we address these challenges by reformulating offline policy evaluation as a liveness problem under the formalism of Markov Decision Process (MDP). This problem formulation yields a value function that naturally accommodates sparse rewards and finite-horizon rollouts with non-monotonic task progression. Building on this formulation, we propose a two-stage policy evaluation framework designed to accurately capture task progression and reduce truncation bias. In the first stage, we compute reliable value estimates for states whose outcomes are unambiguous—namely, states that precede observed task completion. In the second stage, we bootstrap the evaluation of the remaining states using the value function from the first stage. Intuitively, this facilitates the backward propagation of value estimates from successful states to their predecessors within truncated episodes. This mechanism effectively overwrites the default pessimistic assumption applied to time-outs, replacing it with value estimates grounded in observed success.

We evaluate our method on two simulated manipulation tasks using visuomotor policies, including a Vision–Language–Action (VLA) model and a diffusion policy, and a hardware task using human demonstrations. Empirical results show that our approach recovers value functions that more accurately reflect task progression and substantially reduce truncation bias, outperforming classical baselines such as TD(0) and Monte Carlo policy evaluation. The contributions of this work are: 1) a liveness-based formulation of offline policy evaluation for sparse-reward manipulation tasks and 2) a two-stage, bootstrapped evaluation algorithm that mitigates truncation bias in finite-horizon offline data.

\section{Background}
\subsection{Markov Decision Process}\label{subsec:mdp}
Markov Decision Process (MDP) is the formalism that underlies modern Reinforcement Learning (RL) theory \cite{rl_book}, and we heavily utilize the framework in this paper. Formally, a MDP $\mdp$ is a tuple $\mdp =  \left(\sspace, \aspace, \tdyn, R, \dfac\right)$, where $\sspace$ and $\aspace$ are the state space and action space, respectively. $P:\sspace\times\sspace\times\aspace\rightarrow[0,1]$ is the transition probability. $R:\sspace\rightarrow\reals$ is the reward function, and $\dfac \in (0,1)$ is the discount factor.

\subsection{Liveness Problem}
In this section, we introduce the notion of liveness. Suppose we have a system described by continuous-time dynamics
\begin{equation}
    \dot{\state} = \dyn(\state,\action)
\end{equation}
where $\state\in\sspace$ and $\action\in\aspace$ are the state and control of the system, respectively, and $\dot{\state}$ denotes the time derivative of state $\state$. 

Intuitively, the problem of liveness concerns with whether the system can ever reach a specific set of states $\tset$ \cite{lygeros_min_time_ctrl}. We call this region of interest the \emph{target set}. Correspondingly, we define the target function $\tfunc:\sspace\rightarrow\reals$ that encodes how far a state $\state$ is away from $\tset$. $\tfunc(\cdot)$ is defined such that for state $\state\in\tset$, $\tfunc(\state)\leq 0$. The liveness problem formulation is particularly popular in the optimal control community, with the goal being synthesis of controllers that can get the system to reach as far into the target set as possible \cite{reachability_tutorial}. However, in this work, we are primarily interested in the liveness problem formulation for evaluation purposes (i.e., determining how far a given policy can reach into the target set), rather than controller/policy synthesis.

\section{Related Work}
\subsection{Classical Policy Evaluation}Classical approaches to policy evaluation estimate the value function $V_\pi(s)$ as the expected discounted return, $G_t = \sum_{k=0}^{\infty} \gamma^k r_{t+k+1}$ \cite{rl_book}. If the transition dynamics $p(s',r|s,a)$ are known, $V_\pi(s)$ can be computed via Iterative Policy Evaluation –  iteratively applying the following Bellman operator: 
$$V_{\pi}(s)=\sum\limits_{a} \pi(a \mid s) \sum\limits_{s^{\prime},r} p(s^{\prime},r \mid s,a)[r+ \gamma V_{\pi}(s^{\prime})]$$

In offline settings where the transition dynamics is unavailable, Monte Carlo (MC) methods \cite{Rubinstein_mc_book} approximate the value function by averaging the observed returns $G_{i,s}$ from the dataset $\mathcal{D}$ using number of visits $N(s)$ to a state: 
$$V(s) \approx \frac{1}{N(s)} \sum_{i=1}^{N(s)} G_{i,s}$$
Standard MC treats this as a regression problem. However, recent large-scale robotic policies like $\pi_{0.6}^{*}$ \cite{pi_star_06} stabilize this by casting value learning as a multi-class classification problem over discretized return bins, learning a distribution over returns rather than a mean. 

Alternatively, Temporal Difference (TD) learning \cite{sutton_td_learning_1988, td_lambda_convergence} reduces the high variance of MC by updating estimates incrementally via bootstrapping. Instead of waiting for a full episode return, TD updates $V_\pi(s)$ using the immediate reward and the current estimate of the next state's value: 

$$V(s) \leftarrow V(s) +\alpha[r + \gamma V(s') - V(s)]$$

While these methods are foundational, they rely on summing rewards over time. In sparse-reward offline settings, MC suffers from high variance due to limited samples, while TD methods often struggle to propagate sparse success signals backward over long horizons due to vanishing gradients.

\section{Problem Formulation}
In this work, we are interested in evaluating the performance of policies by computing their value functions, or in Reinforcement Learning parlance, the \emph{policy evaluation problem} \cite{rl_book}. Intuitively, our goal is to determine whether the policy can complete the given task. Furthermore, we focus on the \emph{offline} setting where we only have access to episodes generated by the policy but not the policy itself, and we only have access to \emph{sparse rewards} (i.e., the only information provided is whether the task is complete at a given state). Though the problem formulation applies broadly, we are especially interested in manipulation policies. 

More formally, given a task $\task$ and a dataset $\dataset = \{(\state^{i}_{1:\tidxgoal_i}, \action_{1:\tidxgoal_i-1}^i)\}_{i=1}^M$, which contains $M$ episodes each of length less than or equal to $T$ generated by a stochastic policy $\policy$ designed to complete $\task$, we aim to determine the value of any state $\state$, under the given sparse rewards.

\section{Method}

\subsection{Policy Evaluation as Liveness Problem}\label{subsec: policy_eval_liveness_mdp}

We consider policy evaluation in sparse-reward tasks where the only reward signal is whether a state is a goal state (i.e., task completion). Manipulation tasks commonly resemble such a task structure. In such settings, value functions computed using cumulative sparse reward formulations struggle to capture the proper progression of the task. A useful measure of task progression is \textit{semantic} closeness to goal completion. Consider our first case study in Fig.~\ref{fig:teaser}~(a), a manipulation task where the policy is permitted to attempt actions until task completion, a setting that is common in real life manipulation tasks. The episode does not exhibit a monotonic task progression, a structure that cumulative sparse reward formulation inherently assumes. The states becomes semantically closer to task completion as the manipulator grasps the object, but upon slipping and dropping the object, the states become farther separated from task completion. We therefore reformulate policy evaluation as a \textit{discounted} liveness problem, where the value of a state reflects the semantic distance between itself and task completion; i.e., a distance metric in the state space $\mathcal{S}$. 

We first define the \textit{undiscounted} liveness problem using the formalism of Markov Decision Process (MDP). We modify the classical MDP formulation introduced in Sec.~\ref{subsec:mdp} by replacing the reward function $R(\cdot)$ with the target function $\tfunc(\cdot)$: $\mdp = \left(\sspace, \aspace, \tdyn, \tfunc, \dfac\right)$. Intuitively, the target function $\tfunc(s)$ is a signed distance function of the state to some target set $\tset$. We define the value of a state $\state$ by 
\begin{equation}\label{eq:vfunc_liveness}
    \vfuncpi(\state) = \expectation_{\xi\sim\statetrajspace}\left[ \min_{\tidx\in\{0,1,\ldots, \tidxgoal\}} \tfunc\Big(\xi(\tidx)\Big)\right]
\end{equation}
where $\xi$ is a state trajectory (or equivalently, an episode) rolled out from state $\state$ under policy $\pi$ and transition dynamics $\tdyn$, and $\statetrajspace$ is the distribution of all such state trajectories. We assume that all episodes are able to achieve the minimum target over $\tidxgoal$ steps (i.e., extending the episodes by continuing to roll out the policy will not lead to a lower target, or equivalently, pushing further into the target set $\tset$). The value of a state $\state$ is the expectation of the lowest target achieved among all possible trajectories rolled out from $\state$. 

We use the principle of dynamic programming to transform the definition of the value function in Eq.~\ref{eq:vfunc_liveness} into a recursive Bellman equation in Eq.~\ref{eq:vfunc_liveness_ineq}.
\begin{proposition}\label{prop:vfunc_liveness_ineq}
The value function in Eq.~\ref{eq:vfunc_liveness} satisfies the following inequality: 
\begin{equation}\label{eq:vfunc_liveness_ineq}
        \vfuncpi(\state) \leq \expectation_\policy \bigg[ \expectation_\tdyn \Big[\min\big\{\tfunc(\state), \vfuncpi(\state') \big\}\Big] \bigg]
\end{equation}
Furthermore, $|\vfuncpi(\state)- \expectation_\policy \bigg[ \expectation_\tdyn \Big[\min\big\{\tfunc(\state), \vfuncpi(\state') \big\}\Big] \bigg]| \leq \sigma_Y$, where $\sigma_Y$ is the expectation over action $\action_0$ and state $\state_1$ of standard deviation of the random variable $Y = \min_{\tidx\in\{1,\ldots, \tidxgoal\}} \tfunc\big(\xi(\tidx)\big)$. 
\end{proposition}

This proposition is proved in Appx.~\ref{appx:proof_liveness_update_ineq}. The inequality in Eq.~\ref{eq:vfunc_liveness_ineq} implies that any fixed point of the corresponding equality defines a conservative (pessimistic) estimate of the true liveness value. This gap results from applying Jensen's inequality in the derivation of Eq.~\ref{eq:vfunc_liveness_ineq}. Intuitively, when future trajectories have low variability (e.g., near-deterministic dynamics or policy), the gap is small, and the approximation is tight. Conversely, high variability in future trajectories widens the gap, making the approximation more pessimistic.

Since over-stating task progress (i.e., predicting a lower value than warranted) is undesirable in practice, we deliberately adopt this conservative (pessimistic) estimate in Eq.~\ref{eq:vfunc_liveness_overestimate}:

\begin{equation}\label{eq:vfunc_liveness_overestimate}
    \vfuncpiover(\state) = \expectation_\policy \bigg[ \expectation_\tdyn \Big[\min\big\{\tfunc(\state), \vfuncpiover(\state') \big\}\Big] \bigg]
\end{equation}

\subsection{Practical Approximation of the Liveness Value Function from Offline Data}

In our setting of offline policy evaluation, we approximate the distributions of the policy $\pi$ and transition dynamics $\tdyn$ by the empirical distribution of the dataset $\dataset$ generated using $\pi$. More precisely, 
\begin{align}
    \vfuncpiover(\state) &= \expectation_\policy \bigg[ \expectation_\tdyn \Big[\min\big\{\tfunc(\state), \vfuncpiover(\state') \big\}\Big] \bigg] \notag \\
    &\approx \expectation_\dataset \Big[\min\big\{\tfunc(\state), \vfuncpiover(\state') \big\}\Big] \label{eq:liveness_update_empirical}
\end{align}
Then, following Eq.~\ref{eq:liveness_update_empirical}, we arrive at the following Bellman operator 
\begin{equation}\label{eq:liveness_update}
    \vfuncpiover(\state) = \min\big\{\tfunc(\state), \vfuncpiover(\state') \big\}
\end{equation}

Equation ~\ref{eq:liveness_update} does not lend itself naturally to the notion of task progression, as it only captures the lowest value along the trajectory. Thus, a discount factor $\gamma$ is introduced, and consequently induces a contraction mapping \cite{fisac_bridging_2019}, and the result Prop.~\ref{prop:discounted_liveness_update_contraction} is proved in Appx.~\ref{appx:contraction_proof}.

\begin{equation}\label{eq:discounted_liveness_update}
    \vfuncpiover(\state) = (1-\dfac) + \dfac \min\{\tfunc(\state), \vfuncpiover(\state')\}
\end{equation}

\begin{proposition}\label{prop:discounted_liveness_update_contraction}
The discounted Bellman operator Eq.~\ref{eq:discounted_liveness_update} induces a contraction mapping under the supremum norm.
\end{proposition}

The target function $\tfunc(\state)$ is defined using the typical sparse reward for manipulation tasks, where the reward signal is only known at task completion (reaching a \textit{goal state}): 
\begin{equation}\label{eq:lx_def}
    \tfunc(\state) = \begin{cases}
        \ \ 1 &\ \text{if $\state$ is a not goal state}\\
        -1 &\ \text{if $\state$ is a goal state}
    \end{cases}
\end{equation}

\noindent
The discounting, along with the sparse reward formulation in Eq.~\ref{eq:lx_def}, leads naturally to the notion of \emph{steps to task completion}. To see this, let us take a state trajectory $\xi_{\state_0}^{\policy, \tdyn} = \{\state_0, \state_1, \ldots \state_\tidxgoal, \state_{\tidxgoal+1} \ldots \}$. Without loss of generality, suppose $\tfunc(\state_\tidxgoal) = -1$ and $\tfunc(\state_k) = 1 \ \forall k \neq \tidxgoal$. Then using the discounted fixed-point Bellman equation Eq.~\ref{eq:discounted_liveness_update}, we have the following 
\begin{align}
    \vfuncpiover(\state_{\tidxgoal - 1}) &= (1-\dfac) + \dfac\vfuncpiover(\state_\tidxgoal) \notag \\
    &= (1-\dfac) + \dfac(-1) \notag \\
    &= 1 - 2\dfac \notag \\
    \vfuncpiover(\state_{\tidxgoal - 2}) &= (1-\dfac) + \dfac\vfuncpiover(\state_{\tidxgoal-1}) \notag\\
    &= (1-\dfac) + \dfac(1 - 2\dfac) \notag \\
    &=  1 - 2\dfac^2\notag \\
    &\ \ \vdots \notag \\
    \vfuncpiover(\state_{\tidxgoal-k}) &= (1-\dfac) + \dfac\vfuncpiover(\state_{\tidxgoal-k+1}) \notag \\
    &= 1 - 2\dfac^{k}\  \forall \ 0\leq k\leq \tidxgoal \label{eq:sol_to_recur_relation}
\end{align}

\noindent
Eq.~\ref{eq:sol_to_recur_relation} is the solution to the recurrence relation in Eq.~\ref{eq:liveness_recur_relation}
\begin{equation}\label{eq:liveness_recur_relation}
    \vfuncpiover(\state_{\tidxgoal-k}) = (1-\dfac) + \dfac\vfuncpiover(\state_{\tidxgoal-k+1}),\ \vfuncpiover(\state_{\tidxgoal}) = -1
\end{equation}
The number of steps to task completion, denoted by $k$, from state $\state_{\tidxgoal-k}$ is related to its value $\vfuncpiover(\state_{\tidxgoal-k})$ by $k = \log_\dfac\frac{1-\vfuncpiover(\state_{\tidxgoal-k})}{2}$. However, with the current sparse reward target function $\tfunc(s)$ in Eq.~\ref{eq:lx_def}, we only capture monotonic task progression.

\subsection{Refining the Target Function via Bootstrapping}\label{subsec:bootstrap}

\begin{figure}[t!]
    \centering
    \includegraphics[width=0.9\linewidth]{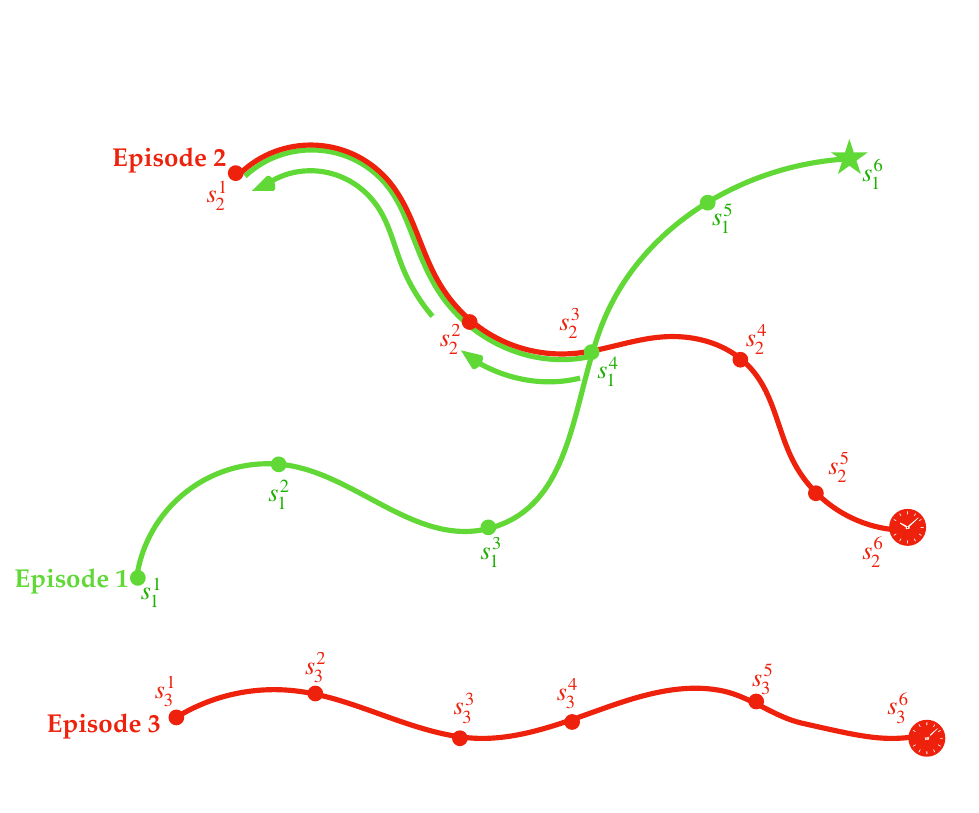}
    \caption{Visualization of the bootstrap mechanism introduced in Sec.~\ref{subsec:bootstrap}. Episode 1 contains 2 goal states, and all states in this episode have well-defined values (i.e., $\vfuncpiwd(\state_1^k)<1 \ \forall \ 0<k\leq6$). Episode 2 is terminated by time-out, but one of its states $\state_2^3$ intersects with Episode 1. Lastly, Episode 3 is terminated by time-out and does not intersect with any episodes with goal states. In this case, all of its states take on the default sparse reward value of 1.}
    \label{fig:bootstrap}
\end{figure}

In order to fully capture the non-monotonic task progression, the target function $\tfunc(s)$ must provide a notion of distance to the target set, which is not provided by the binary sparse reward signal. We identify states where the policy’s performance is unambiguous. Let $\dataset'\subseteq\dataset$ be the set of success episodes (reaches a goal state), and $\dataset\backslash\dataset' \subseteq \dataset$ be the set of timed-out episodes (never reaches any goal states). We compute the values of all descendant states of all goal states within $\dataset'$ using the Bellman operator in Eq.~\ref{eq:discounted_liveness_update}, and the values of these states are said to be \emph{well-defined}. These values serve as ``anchors", as they are computed using observed success from goal states. Let value function $\overline{\vfuncpi}$ take on the values of the states with well-defined values, and for all states $\state$ without well-defined values, $\overline{\vfuncpi}(\state) = 1$. 

Next, we bootstrap the target function $\tfunc(\state)$ with $\vfuncpiwd(\state)$ to provide a more accurate signal for rough distance to the target set. For example, states $\left\{s_1^{1},...s_1^5\right\}$ were initially pessimistically labeled with $l(s)=1$, but their updated targets are $\vfuncpiwd(\state)<1$. We apply the Bellman operator in Eq.~\ref{eq:discounted_liveness_update} to the dataset $\dataset$. For clarity, the Bellman operator with the bootstrapped $\tfunc(\state)$ is as follows
\begin{equation}\label{eq:liveness_update_with_z_bootstrap}
    \vfuncpiover(\state) = (1-\dfac) + \dfac \min\{\vfuncpiwd(\state), \vfuncpiover(\state')\}
\end{equation}

Importantly, the $\min$ operator in Eq.~\ref{eq:discounted_liveness_update} acts as a ``correction filter". We illustrate the mechanism in Fig.~\ref{fig:bootstrap}. Let us take Episode 2, denoted by $\xi_2$, an timed-out episode (without any goal states), and we take state $\state_2^3\in\xi_2$. Since $\xi_2$ is terminated at state $\state_2^6$, using Eq.~\ref{eq:discounted_liveness_update}, we compute that $\vfuncpiover(\state_2^6) = \vfuncpiover(\state_2^5) = \vfuncpiover(\state_2^4) = 1$. However, $\state_2^3$ is a state in a successful episode ($\state_1^4$ in Episode 1), via bootstrapping, the known well-defined value $\vfuncpiwd(\state_2^3)$ immediately overrides the pessimistic time-out assumption of $\vfuncpiover(\state_2^4) = 1$ as $\min\{\vfuncpiwd(\state_2^3), \vfuncpiover(\state_2^4)\} = \vfuncpiwd(\state_2^3)$. Furthermore, for the predecessor state $\state_2^2$ of $\state_2^3$, though $\state_2^2$ is not a state of any successful episodes (i.e., $\vfuncpiwd(\state_2^2) = 1$), the corrected value of $\vfuncpiover(\state_2^3)$ will be propagated backward via $\min\{\vfuncpiwd(\state_2^2), \vfuncpiover(\state_2^3)\} = \vfuncpiover(\state_2^3)$. The corrected value $\vfuncpiover(\state_2^2)$ will be similarly propagated backward to the initial state $\state_2^1$ of $\xi_2$. On the other hand, none of the states in Episode 3 can have their value corrected so they will take on the given pessimistic sparse reward value of 1. 

Effectively, an episode is only penalized as a failure if it neither contains goal states nor overlaps with any episode that does. We summarize our overall method in Algo.~\ref{algo:our_method}.

\begin{remark}
    In practice, the state space is continuous, and as a result we would never encounter two identical states in the dataset. But the intuition holds that if two states are similar enough, the bootstrapping mechanism will trigger as if they are identical. This is further discussed in detail in Sec.~\ref{sec:discussion}. 
\end{remark}

\begin{algorithm}\label{algo:our_method}
\DontPrintSemicolon
  
  \KwInput{Dataset $\dataset$}
  \KwOutput{Value function $\vfuncpiover$}
  
  Compute $\overline{\vfuncpi}$ using subset $\dataset'$ of dataset $\dataset$\;
  Bootstrap $\overline{\tfunc(\state)}$ using $\overline{\vfuncpi}$\;
  Compute $\vfuncpiover$ using $\overline{\tfunc(\state)}$ and Eq.~\ref{eq:discounted_liveness_update}\;
  \Return{$\vfuncpiover$}\;
  
\caption{Two-Stage Policy Evaluation Framework}
\end{algorithm}

\subsection{Reduction of Truncation Bias from Manual Time-outs}

We have established a value function that is capable of capturing non-monotonic task progression often seen in manipulation tasks, and a major benefit of the bootstrapping mechanism is the reduction of \textit{truncation bias}. Truncation bias is introduced when episodes that are terminated via time-out are pessimistically assumed to have timed-out in a way that the policy will never reach the goal, even if given more time. Distinguishing between policy failures and episode truncations is a key practical challenge in policy evaluation, and the resulting bias leads to value functions that underestimate policy performance.

To see how our formulation reduces the truncation bias, we take Episode 2 in Fig.~\ref{fig:bootstrap}. Since it is manually timed-out, all of its states would have taken on the pessimistic value of $1$ without the bootstrapping mechanism. However, the bootstrapping mechanism updates the values for states $\state_2^3$, $\state_2^2$, and $\state_2^1$ to optimistic values that more accurately reflect the actual task progression, effectively reducing truncation bias. 

\section{Experiments}

\begin{figure*}[t]
    \centering
    \includegraphics[width=\textwidth]{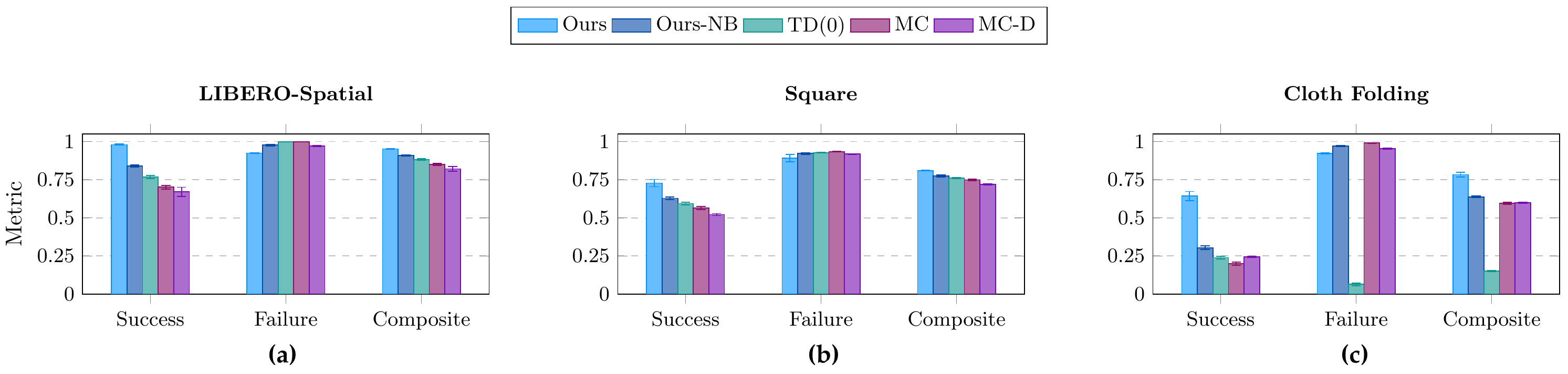}
    \caption{Success, failure, and composite metrics for each method on the (a) LIBERO-Spatial, (b) Square, and (c) Cloth Folding tasks (dataset size 200 for the simulation experiments and 150 for the hardware experiment, 5 training seeds, higher is better; error bars denote standard deviation).}
    \label{fig:metrics_results_plot}
\end{figure*}

We assess the performance of our method, along with several baselines, in 3 case studies: a pick and place task from LIBERO \cite{libero}, a square peg and hole insertion task from Robomimic \cite{robomimic2021} (referred to as the Square task), and a cloth folding task on a Franka Panda hardware platform. Each case study utilizes a dataset generated by the target policy consisting of observations consumed by the policy and the success/timed-out binary label. All methods perform offline policy evaluation by estimating the value function of the target policy.

\vspace{0.5em}
\noindent \textbf{Baselines.} We compare our method against four baselines: (1) \textbf{Ours-NB}, an ablation of our method without bootstrapping the target function $\tfunc(\cdot)$; (2) \textbf{TD}, a standard TD(0) implementation \cite{sutton_td_learning_1988}; (3) \textbf{MC}, a Monte Carlo baseline \cite{rl_book}; and (4) \textbf{MC-D}, the Distributional Monte Carlo implementation from \cite{pi_star_06}. The baseline Ours-NB uses the sparse reward from Eq.~\ref{eq:lx_def}, while the remaining baselines use the following equivalent sparse reward from \cite{pi_star_06}:
\begin{equation}
    r_t = \begin{cases} 
0 & \text{if } t = T \text{ and success} \\
-C_{\text{fail}} & \text{if } t = T \text{ and failure} \\
-1 & \text{otherwise}.
\end{cases}
\end{equation}
where $C_{\text{fail}}$ is a large constant. 

\vspace{0.5em}
\noindent 
\textbf{Value function learning. }All the value functions are parameterized as neural networks. The inputs to the neural networks are the concatenation of SigLIP2 \cite{siglip2} latent embeddings of image observations and proprioception states, matching the input space of the evaluated policy. Additionally, we utilize prioritized experience replay \cite{per} during training. The hyperparameters for the neural networks, training, and prioritized experience replay are provided in Appx.~\ref{appx:case_study_hyperparameters}.

\vspace{0.5em}
\noindent 
\textbf{Evaluation metrics.} We consider 3 metrics. The first metric is the \textbf{success metric}. Given a (sub)episode of length $\tidxgoal$, where the target policy successfully reaches a goal state, we convert the value at each frame into an estimated number of steps to the goal state. The metric is defined as the percentage of frames where the estimated step count is less than $\tidxgoal$. A high percentage indicates that the model correctly predicts that the goal can be reached within the actual episode duration $\tidxgoal$.

The second metric, referred to as the \textbf{failure metric}, assesses the value function's ability to recognize frames associated with timed-out episodes (i.e., containing no goal states). Given a failure episode, we compare the estimated steps to success against a standard task horizon $N_s$, a threshold representing the typical maximum number of steps required to solve the task from any starting state. The metric is defined as the percentage of frames where the inferred step to success exceeds $N_s$. A high percentage indicates that the model correctly predicts that the goal is effectively unreachable within the task horizon.

The third metric is the \textbf{composite metric}, which is defined to be the arithmetic mean of the success and failure metrics. We utilize this metric to capture the trade-off between the success and failure metrics. 

\begin{remark}[Relation to standard classification metrics]
The success and failure metrics defined above involve thresholding the value to a binary label, which may suggest that standard classification metrics are applicable in our case studies. However, despite the thresholding, standard classification metrics (e.g., precision, recall, and F1) are not directly interpretable in our setting because binary labels are not well-defined for all states due to recovery behaviors and manual time-outs. Specifically, in successful episodes, intermediate states may exhibit recovery behaviors before eventual task completion, making the notion of steps-to-success ambiguous. Therefore, the success metric is restricted to the final $N$-step subtrajectory where progress is monotonic and steps-to-success are well-defined. Earlier states in the same episode are excluded from evaluation. The failure metric is defined as the percentage of states in failure episodes whose predicted steps-to-success exceeds $N_s$, since these states are empirically unlikely to lead to task completion within $N_s$ steps. Furthermore, a state not classified as successful does not necessarily imply failure — it may still be recoverable. As a result, the two metrics are defined over disjoint subsets of states, and there is no consistent binary labeling over the full dataset, making standard classification metrics unsuitable for this problem setting. 
\end{remark}

\subsection{Case Study 1: LIBERO Pick and Place}\label{subsec: case_study_libero}

The target policy for this case study is Physical Intelligence's $\pi_0$ \cite{pi0} with a fixed language command of ``\textit{pick up the black bowl next to the ramekin and place it on the plate}". The task is part of the LIBERO-Spatial task suite \cite{libero}, a kitchen table-top simulation environment with various objects where a Franka Panda arm attempts to pick up a bowl and put in on a plate. We collected 200 training and 100 testing episodes, each with $50\%$ success/timed-out split. The manual time-out is set to 250 steps for this case study.

We first provide a qualitative example of our method in action in Fig.~\ref{fig:teaser}~(a). Intuitively, a value of $-1$ corresponds to imminent success, and value of $1$ corresponds to requiring infinite steps to complete the task (i.e., a definitive failure). When the episode starts, the gripper is very close to the bowl, and our method assigns a moderate negative value. The manipulator picks up the bowl and the value decreases since the policy is making good progress. Then, the bowl slips out of the gripper, and the value increases sharply, correctly indicating that the state is now significantly farther from the goal, semantically. Finally, the policy recovers by picking up the bowl again and completes the task. Correspondingly, the value tends toward $-1$ as the policy progresses towards task completion. 

We now discuss the quantitative results for the 3 metrics presented in Fig.~\ref{fig:metrics_results_plot}~(a). Our method achieves the best performance in both the success metric and the composite metric, while remaining competitive in the failure metric. Notably, the policy in this experiment is capable of recovering from various setbacks to achieve success if more time is given. Our bootstrapping mechanism is able to correct the pessimistic values of many states in timed-out episodes to reflect a more accurate value for the state given the policy's capabilities. This largely explains the wide margin in the success metric between our method and the baselines.

\subsection{Case Study 2: Square Peg and Hole Insertion}
In this case study, we evaluate a diffusion policy \cite{chi2023diffusionpolicy} trained to perform a peg and hole insertion task simulated in Robomimic \cite{robomimic2021}. The task involves a Franka Panda arm grasping a square nut and placing it on a square peg. We again collected 200 and 100 episode datasets with 50\% success splits each, for training and testing, respectively. The manual time-out is set to 200 steps.

We present the quantitative results in Fig.~\ref{fig:metrics_results_plot}~(b), where our method achieves the best performance in the success metric while remaining competitive in the failure metric. However, both our method and the baselines perform substantially worse in the success metric compared to the LIBERO-Spatial case study. Unlike the VLA-driven LIBERO-Spatial task, the diffusion policy is trained on a single task without expert recovery behavior, meaning unrecoverable failures occur in the episodes at unknown times. Additionally, there are fewer visual cues in the image, which is dominated primarily by the static background. Thus, successful trajectories are often visually identical to timed-out episodes until a setback occurs (e.g., a sudden slip, poor grasp, or bad peg alignment). Hence, the value function is unlikely to benefit from the bootstrapping mechanism when many opposite-label states nearly overlap in the image space, or have similar semantic meaning. This noise in the image-label mapping may result in the decreased success metric score, as the value function is uncertain that the goal can be reached within the task horizon until a strong visual cue (e.g., good grasp, proper alignment) has been observed.

One emergent property of the value function is its predictive capability. We share one example in Fig.~\ref{fig:teaser}~(b). In this episode, the gripper does not align well with the square nut leading up to the grasp, and at the time of the grasp being established, the value is already close to 1, strongly indicating that success is unlikely. A few steps later, the square nut slips out of the gripper, and the policy never recovers. The value function is able to associate the misaligned grasp with the subsequent slippage. This predictability could be due to the strong semantic cue for a poor grasp versus a proper grasp, which is correlated with task success. 

On the flip side, the value function is not able to preemptively capture failure moments that are ``unexpected". For instance, the square nut sometimes slips out of the gripper despite an established good grasp, at least to the human eyes. The value function, at best, can react to the slippage as it happens but not preemptively. Such slippage events may occasionally act in the opposite way, for example, in Fig.~\ref{fig:square_task_fig_379}. The manipulator grasps the square block poorly (an off-center grasp is more often associated with failure), but the grasp happens to slip into a proper grasp after contacting the peg, and the task is completed successfully. This false-positive behavior draws down the performance metrics, but indicates that the value function has learned that an anomalous grasp is likely to lead to failure and adjusts accordingly upon making a proper grasp.

\begin{figure}
    \centering
    \includegraphics[width=1\linewidth]{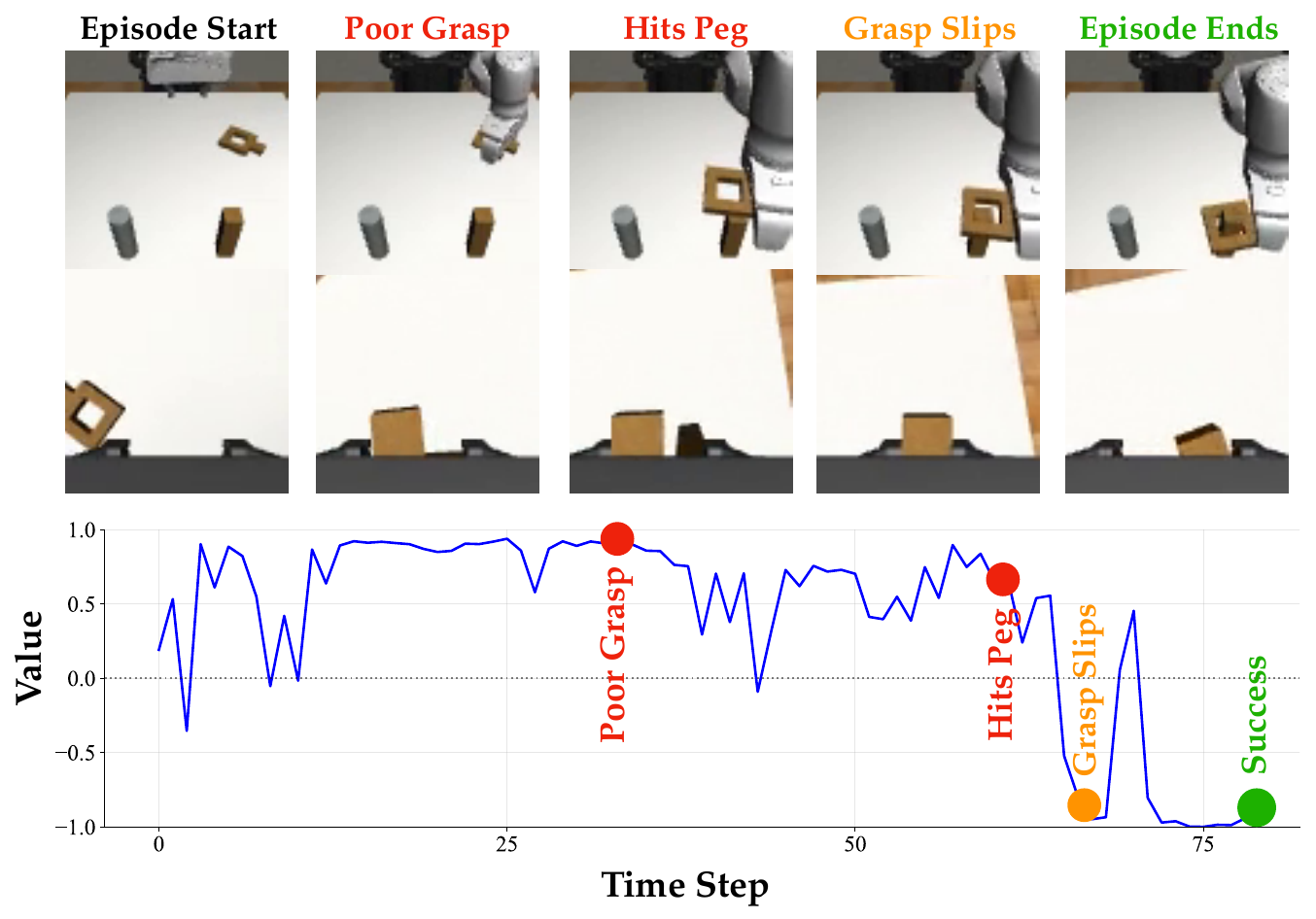}
    \caption{Our method's predicted values over a successful Square task episode showing a false-positive high value due to a poor grasp that managed to succeed.}
    \label{fig:square_task_fig_379}
\end{figure}

\subsection{Case Study 3: Human Cloth Folding}

We conduct a hardware experiment on a cloth folding task using a Franka Panda platform, where the value function is learned from human demonstrations. A human operator teleoperates the robot via a Meta Quest 3 headset and controller, with the goal of folding a rectangular towel into a triangle along one of its diagonals. This setting is substantially more challenging than the simulated tabletop tasks for three reasons: (1) the cloth is deformable and difficult to manipulate, (2) human teleoperation behavior is inconsistent across episodes, and (3) labels are noisy, since human visual inspection sometimes assigns different labels to visually similar end states.

We train all methods on 150 episodes (75 success, 75 timed-out) and evaluate on 20 held-out episodes (10 success, 10 timed-out). The manual time-out is set to 300 steps in this case study. Unlike the simulated experiments, the value function input here consists solely of the encoded latent of the base image, as this is the only observation available to the operator during teleoperation.

We first present a qualitative example in Fig.~\ref{fig:teaser}~(c). For most of the episode, the operator struggles to establish a good grasp on the cloth, and the value oscillates near zero, reflecting stagnant task progression. Once a successful grasp is established and the folding begins, the value drops sharply toward $-1$, indicating clear progress toward task completion.

Quantitative results are shown in Fig.~\ref{fig:metrics_results_plot}~(c). Our method outperforms all baselines on the success metric by a wide margin, more than 2$\times$ the next-best baseline (Ours-NB), while trailing slightly on the failure metric. Two factors contribute to this outcome: (1) recovery behaviors are prevalent throughout the dataset because the operator is not adept at teleoperated cloth manipulation, a setting in which our bootstrapping mechanism is particularly effective; and (2) the operator does not exhibit clear failure modes, and as a result observations from success and timed-out episodes are often visually similar, complicating value function learning.

The cloth folding task represents the most extreme instance of the conditions our method targets: frequent recovery behavior, and unpredictable episode truncation caused by fixed-time time-outs that cut off episodes at varying stages of task progression. The wide success metric margin over baselines suggests that the bootstrapping mechanism is most effective precisely in these settings.

\subsection{Ablating the Bootstrapping}
To isolate the contribution of the bootstrapping mechanism, we consider an ablation of our method, \textbf{Ours-NB}, where the target function is simply the sparse rewards defined in Eq.~\ref{eq:lx_def}. As we show in our experiments, it underperforms compared to its bootstrapped counterpart in all 3 case studies, performing worse on the success and composite metrics but slightly better on the failure metric (Fig.~\ref{fig:metrics_results_plot}). This observation supports the conclusion that the bootstrapping increases success metric accuracy at a small sacrifice of failure metric accuracy.

\subsection{Ablation On Dataset Size}\label{subsec:dataset_size_ablation}

\begin{figure*}[t]
    \centering
    \includegraphics[width=0.9\textwidth]{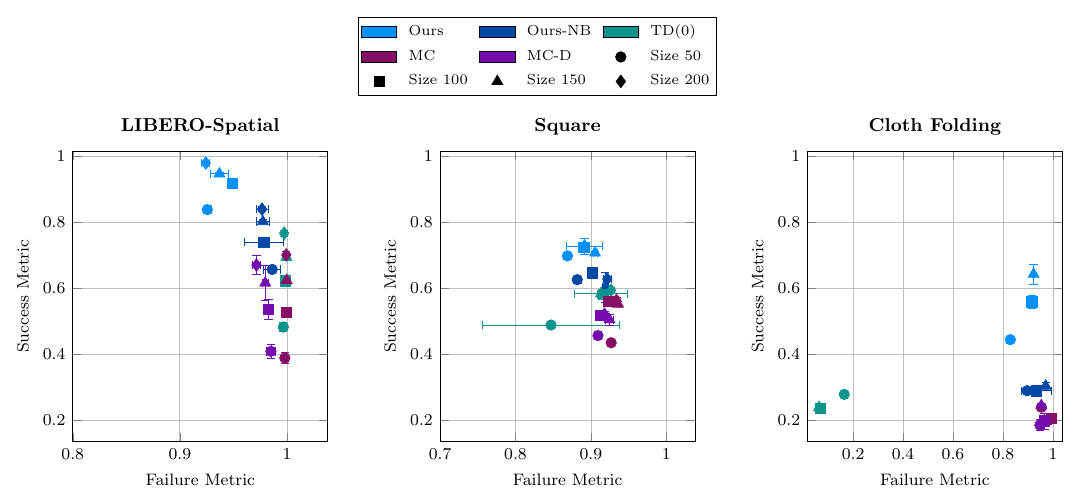}
    \caption{Dataset size ablation showing success metric vs.\ failure metric (higher is better for both) for each method on the (a) LIBERO-Spatial, (b) Square, and (c) Cloth Folding tasks across 5 training seeds. Colors distinguish methods while shapes distinguish dataset sizes; error bars denote standard deviation. Cloth Folding is evaluated only up to dataset size 150.}
    \label{fig:dataset_ablation_fig}
\end{figure*}

Since we are trying to evaluate policies with their actual rollouts, sample efficiency is crucial for the practical feasibility of our method. We conduct an ablation study for each case study to assess the performance of all the methods under different dataset sizes. Recall that we collected 200 episodes in the simulation experiments and 150 episodes in the hardware experiment for training, with 50\%-50\% successful/timed-out episodes split, in each case study. We subsample 50, 100, and 150 episodes, maintaining a 50\%-50\% success/time-out split, and run all the methods on the smaller datasets. 

We visualize the success metric and failure metric results for all 3 case studies in Fig.~\ref{fig:dataset_ablation_fig}. As the dataset size increases, our method improves on both the success and failure metrics, and the variance on the performance decreases. Furthermore, our method outperforms the baseline in the success metric while being competitive in the failure metric, across all dataset sizes. 

Notably, our method with the smallest dataset sizes still maintains its much improved success metric and competitive failure metric when compared to baseline methods with the full dataset. This shows that our method remains useful in low-data regimes.

\subsection{Ablation on Embedding Space}\label{subsec:ablation_embedding}

\begin{figure}[t!]
    \centering
    \includegraphics[width=1\linewidth]{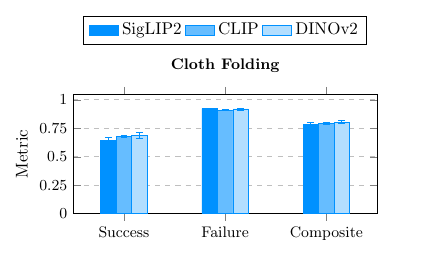}
    \caption{Encoder ablation on the Cloth Folding task (dataset size 150, 5 training seeds, higher is better; error bars denote standard deviation) comparing SigLIP2, CLIP, and DINOv2 as the image encoder for our method.}
    \label{fig:encoder_ablation_results}
\end{figure}

While the default image encoder used in all experiments is SigLIP2~\cite{siglip2}, we additionally conduct an ablation on the choice of image encoder in the cloth folding task with the full dataset (150 episodes). We run our method with SigLIP2~\cite{siglip2}, DINOv2~\cite{oquab2023dinov2}, and CLIP~\cite{clip}, with quantitative results across all metrics shown in Fig.~\ref{fig:encoder_ablation_results}.

Our method performs comparably across the three encoders on all metrics, indicating that it is largely insensitive to the choice of image encoder. This is expected, as large pretrained image encoders are generally competent at retaining task-relevant features. So long as the necessary information is preserved in the latent representation, our method can correlate these features with semantic task progress. Conversely, if the encoder fails to retain such information, we would expect any method, including ours, to degrade.

\subsection{Statistical Significance}

To assess statistical significance, we apply the Alexander-Govern test~\cite{alexander1994} to detect overall differences among methods, followed by pairwise Welch's t-tests~\cite{welch1947} with Benjamini-Hochberg correction~\cite{benjamini1995} to control the false discovery rate, both at significance level $\alpha=0.05$. Across all case studies and dataset sizes, our method achieves statistically significant improvements over all baselines on both the success and composite metrics.

\section{Discussion}\label{sec:discussion}

In this section, we discuss some observations derived from considering all the experiment results across 3 case studies and different dataset sizes. 
\subsection{Qualitative Analysis of Our Method}
Our first observation is that our method consistently achieves the best success metric, and this is expected due to the bootstrapping mechanism described in Sec.~\ref{subsec:bootstrap}. In theory, the bootstrapping mechanism ensures that states on a timed-out episode take on values less than 1 as long as they have a descendant state that is previously seen on a success episode, effectively making the values of those states more optimistic than the default pessimistic value of 1. Since the value function is a learned neural approximator with data from the dataset of episodes, the bootstrapping mechanism does not apply exclusively to the states on timed-out episodes with descendant states seen on other success episodes, but rather states from timed-out episodes that look similar enough to states on success episodes. As a result, a large number of states from timed-out episodes have their values corrected lower (i.e., taking more optimistic values). However, the effectiveness of the bootstrapping mechanism varies depending on the target policy and the task. For policies that exhibit more recovery behaviors, the bootstrapping mechanism is beneficial for capturing the non-monotonic progression of the task, as the timed-out episodes generated from such policies are more severely affected by the truncation bias.  

Our second observation is that the \texttt{min} operator, that corrects the pessimism induced by failure labels and truncation bias, encourages the network to learn non-monotonic task progression. Fig.~\ref{fig:teaser}~(a) shows our method will recognize a recoverable failure (dropping the bowl) and result in an increase in value, indicating it will take more steps to complete the task, until the manipulator adjusts to a position from which it may recover and the value then decreases as it picks up the bowl and completes the task. Fig.~\ref{fig:square_task_fig_379} shows similar behavior, where an anomalous grasp is flagged as a high-value failure-prone behavior until the grasp is adjusted to a proper grasp and the value accordingly decreases toward the task completion. Conversely, Fig.~\ref{fig:teaser}~(b) shows how when an \textit{irrecoverable} failure occurs (the manipulator drops the object and the policy is not trained with recovery behavior from such a state), the value remains as high as possible, i.e., infinite steps to success.

The third observation is that our method consistently achieves the worst failure metric, and this is also expected as a consequence of the bootstrapping mechanism. As we have mentioned in the previous observation, states from timed-out episodes that look similar enough to states on success episodes will take on more optimistic values. In this case, this is to the detriment of our method, as some timed-out episodes should not have their states updated with more optimistic values. But since some of the states look similar enough to some states on success episodes, their values and their predecessors' values get updated to be more optimistic than they should be. This problem is especially severe if the task environment does not contain rich visual cues (i.e., states are not very distinguishable).

\subsection{Practical Considerations}
Based on the first and third observations, it seems like our method essentially renders the value function more optimistic. The reality is more nuanced. Our method is able to maintain a competitive failure metric and simultaneously substantially improve upon the success metric. In some sense, our method is paying a small penalty in the failure metric for the larger improvement on the success metric. It is worthwhile to point out that our method does not push the values of all states to be more optimistic, but it only does so for states that can be closely associated with known success states. Such state association depends on the latent/pixel space similarity for the visuomotor policy input. This soft overlap is likely handled by the neural network learning a smooth manifold in the SigLIP2 feature space, such that perceptually similar inputs lead to similar network outputs; i.e., semantic distance. This may implicitly build a graph similar to the tabular case. This mechanism allows the Bellman update to propagate success signals backward through failure trajectories.

This may explain the performance divergence between the LIBERO-Spatial/Cloth Folding tasks and the Square tasks. In the LIBERO-Spatial and Cloth Folding tasks, semantically similar events (e.g., grasping the bowl/cloth) could occur multiple times in both successful and timed-out episodes, allowing the network to learn the association between this behavior and the label without penalizing recovery attempts. However, in the Square task, no recoveries were attempted (dropping the object is irrecoverable under the policy), so the initial $l(s)$ will tend to be overly conservative, reducing the benefit of the bootstrapping mechanism.

\section{Conclusion} 
\label{sec:conclusion}
We present a framework for offline evaluation of visuomotor manipulation policies with sparse rewards and truncation bias from finite-length episodes. By reformulating policy evaluation as a liveness problem, our method accurately captures task progression and reduces truncation bias.

This work has several limitations. First, the bootstrapping mechanism can suffer from over-optimism because the overlap of states is learned by the network, leaving no definite boundary for deciding which timed-out episodes are considered to have intersected success episodes (as in Fig.~\ref{fig:bootstrap}). As a result, the effectiveness of the bootstrapping mechanism depends heavily on the image/latent inputs to the network, affecting the initial estimation $\vfuncpiwd(\state)$ in the first stage. Second, the performance of our method depends heavily on the dataset, and it is not expected to be reliable outside its training distribution. Depending on the makeup of the dataset, the value function can skew optimistic or pessimistic, and currently it is unclear to us how to characterize such bias systematically. Third, the performance of our method was shown to decrease with a policy that does not exhibit recovery behaviors, which is an expected result from the bootstrapping. 

Future work in this direction could cover the over-optimism from the state overlap problem by implementing some structured boundary defining state similarity. Additionally, improvements in data efficiency could be explored to reduce the burden of policy rollout collection. Evaluating various other visuomotor policy architectures and how their differences manifest in value function performance metrics is another direction to explore.

\section*{Acknowledgments}
This research is supported in part by the DARPA Assured Neuro Symbolic Learning and Reasoning (ANSR) program and by the NSF CAREER program (2240163). The authors would also like to thank Albert Lin for setting up the Franka Panda hardware platform and assisting with the hardware experiments.

\newpage
\bibliographystyle{plainnat}
\bibliography{references}

\cleardoublepage
\newpage

\appendices
\onecolumn
\section{Proof of Proposition \ref{prop:vfunc_liveness_ineq}}\label{appx:proof_liveness_update_ineq}
\begin{proof}
We use a slightly different indexing of states in this proof compared to the statement of the proposition for easier indexing in the proof. Let $\state = \state_0$ and $\state' = \state_1$. 

We start from the definition of value from Eq. \ref{eq:vfunc_liveness}. Note that in Eq.~\ref{eq:liveness_update_ineq_proof_line_traj_exp}, the states $\state_k \ \forall k\in\{0,1,\ldots,\tidxgoal\}$ are states, indexed sequentially, in the state trajectory $\xi$. Throughout the proof, we make use of the law of iterative expectation to break the expectation in Eq.~\ref{eq:liveness_update_ineq_proof_line_full_exp} ultimately into 3 nested expectations in Eq.~\ref{eq:liveness_update_ineq_proof_line_nested_exp}. We make use of Jensen's inequality in Eq.~\ref{eq:liveness_update_ineq_proof_line_jensen_ineq} to move the inner expectation inside the outer $\min$ operator. With $\tfunc(\state_0)$ being a constant, $\min \left\{ \tfunc(\state_0), \min_{\tidx \in \{1, \dots, \tidxgoal\}} \tfunc(\state_\tidx) \right\}$ is concave with respect to $Y = \min_{\tidx \in \{1, \dots, \tidxgoal\}} \tfunc(\state_\tidx)$, and hence with Jensen's Inequality we have $\expectation\left[f(Y)\right]\leq f(\expectation\left[Y\right])$. Detailed derivations are as followed. 

\begin{align}
    \vfuncpi(\state_0) &= \expectation_{\xi \sim \Xi_{\state_0}^{\policy,\tdyn}} \min_{\tidx \in \{0, 1, \dots, \tidxgoal\}} \tfunc(\state_\tidx)\label{eq:liveness_update_ineq_proof_line_traj_exp} \\
    &= \expectation_{\action_0, \state_1, \action_1, \dots, \state_{\tidxgoal}} \min_{\tidx \in \{0, 1, \dots, \tidxgoal\}} \tfunc(\state_\tidx) \label{eq:liveness_update_ineq_proof_line_full_exp}\\
    &= \expectation_{\action_0 \sim \policy} \expectation_{\state_1, \action_1, \dots, \state_{\tidxgoal} | \action_0} \min_{\tidx \in \{0, 1, \dots, \tidxgoal\}} \tfunc(\state_\tidx) \\
    &= \expectation_{\action_0 \sim \policy} \expectation_{\state_1 \sim \tdyn} \expectation_{\action_1, \state_2, \dots, \state_\tidxgoal | \action_0, \state_1} \min_{\tidx \in \{0, 1, \dots, \tidxgoal\}} \tfunc(\state_\tidx) \label{eq:liveness_update_ineq_proof_line_nested_exp} \\
    &= \expectation_{\action_0 \sim \policy} \expectation_{\state_1 \sim \tdyn} \expectation_{\action_1, \state_2, \dots, \state_\tidxgoal|\action_0, \state_1} \min \left\{ \tfunc(\state_0), \min_{\tidx \in \{1, \dots, \tidxgoal\}} \tfunc(\state_\tidx) \right\} \\
    &\le \expectation_{\action_0 \sim \policy} \expectation_{\state_1 \sim \tdyn} \min \left\{ \tfunc(\state_0), \expectation_{\action_1, \state_2, \dots, \state_\tidxgoal|\action_0, \state_1} \min_{\tidx \in \{1, \dots, \tidxgoal\}} \tfunc(\state_\tidx) \right\} \label{eq:liveness_update_ineq_proof_line_jensen_ineq} \\
    &= \expectation_{\action_0 \sim \policy} \expectation_{\state_1 \sim \tdyn} \min \left\{ \tfunc(\state_0), \vfuncpi(\state_1) \right\}
\end{align}

We now show $|\vfuncpi(\state)- \expectation_{\action_0 \sim \policy} \expectation_{\state_1 \sim \tdyn} \min \left\{ \tfunc(\state_0), \vfuncpi(\state_1) \right\}| \leq \sigma_Y$. We first note that $\vfuncpi(\state) \leq \expectation_{\action_0 \sim \policy} \expectation_{\state_1 \sim \tdyn} \min \left\{ \tfunc(\state_0), \vfuncpi(\state_1) \right\}$, and as a result we only need to show $ \expectation_{\action_0 \sim \policy} \expectation_{\state_1 \sim \tdyn} \min \left\{ \tfunc(\state_0), \vfuncpi(\state_1) \right\} - \vfuncpi(\state) \leq \sigma_Y$. We first denote $G = \expectation_{\action_0\sim\policy}\expectation_{\state_1\sim \tdyn} \min \{\tfunc(\state_0), \vfuncpi(\state_1)\} - \vfuncpi(\state_0)$. Then, 

\begin{align}
    G &= \expectation_{\action_0, \state_1} \min \{\tfunc(\state_0), \vfuncpi(\state_1)\} - \vfuncpi(\state_0) \\
    &= \expectation_{\action_0, \state_1} \min \{\tfunc(\state_0), \vfuncpi(\state_1)\} - \expectation_{\action_0, \state_1, \action_1, \ldots} \min\{\tfunc(\state_0), Y\}\\
    &= \expectation_{\action_0, \state_1} \big[\tfunc(\state_0)  - \max\{0, \tfunc(\state_0) - \vfuncpi(\state_1)\}\big] - \expectation_{\action_0, \state_1, \action_1, \ldots} \big[\tfunc(\state_0) - \max\{0, \tfunc(\state_0) - Y\}\big] \label{eq:jensen_gap_proof_g_min_max_identity} \\
    &= \expectation_{\action_0, \state_1, \action_1, \ldots} \max\{0, \tfunc(\state_0) - Y\} - \expectation_{\action_0, \state_1} \max\{0, \tfunc(\state_0) - \vfuncpi(\state_1)\} \label{eq:jensen_gap_proof_g_max}
\end{align}
Note that we use the identity $\min\{a,b\} = a - \max\{0, a-b\}$ in Eq.~\ref{eq:jensen_gap_proof_g_min_max_identity}, and we take $\tfunc(\state_0)$ out of each expectation in Eq~\ref{eq:jensen_gap_proof_g_max}. We consider 3 cases. 

\textit{Case 1: $\tfunc(\state_0) < \tfunc(\state_k)$ for all possible future states $\state_1, \ldots, \state_N$.}
Then $\tfunc(\state_0) - Y < 0$ for any state trajectories, and as a result $\expectation_{\action_0, \state_1, \action_1, \ldots} \max\{0, \tfunc(\state_0) - Y\} = 0$. Similarly, $\tfunc(\state_0) - \vfuncpi(\state_1) < 0$ for any $\action_0$ and $\state_1$. Hence, 
$G = \expectation_{\action_0, \state_1, \action_1, \ldots} \max\{0, \tfunc(\state_0) - Y\} - \expectation_{\action_0, \state_1} \max\{0, \tfunc(\state_0) - \vfuncpi(\state_1)\} = 0 - 0 = 0$.

\textit{Case 2: $\tfunc(\state_0) > \tfunc(\state_k)$ for all possible future states $\state_1, \ldots, \state_N$.} Then, we have $\expectation_{\action_0, \state_1, \action_1, \ldots} \max\{0, \tfunc(\state_0) - Y\} =\tfunc(\state_0) = \expectation_{\action_0, \state_1} \max\{0, \tfunc(\state_0) - \vfuncpi(\state_1)\}$. Hence, $G = \expectation_{\action_0, \state_1, \action_1, \ldots} \max\{0, \tfunc(\state_0) - Y\} - \expectation_{\action_0, \state_1} \max\{0, \tfunc(\state_0) - \vfuncpi(\state_1)\} = 0$. 

\textit{Case 3: Otherwise}. We continue the derivation from Eq.~\ref{eq:jensen_gap_proof_g_max}. 
\begin{align}
    G &= \expectation_{\action_0, \state_1, \action_1, \ldots} \max\{0, \tfunc(\state_0) - Y\} - \expectation_{\action_0, \state_1} \max\{0, \tfunc(\state_0) - \vfuncpi(\state_1)\} \\
    &= \expectation_{\action_0, \state_1}\Big[ \expectation_{\action_1, \state_2, \ldots | \state_1} \max\{0, \tfunc(\state_0) - Y\} - \max\{0, \tfunc(\state_0) - \vfuncpi(\state_1)\} \Big] \\
    &= \expectation_{\action_0, \state_1}\Big[ \expectation_{\action_1, \state_2, \ldots | \state_1} \big[\max\{0, \tfunc(\state_0) - Y\}\big] - \max\{0, \tfunc(\state_0) - \expectation_{\action_1, \state_2, \ldots | \state_1} Y\} \Big] \\
    &= \expectation_{\action_0, \state_1} \expectation_{\action_1, \state_2, \ldots | \state_1} \Big[\max\{0, \tfunc(\state_0) - Y\} - \max\{0, \tfunc(\state_0) - \expectation_{\action_1, \state_2, \ldots | \state_1} Y\} \Big] \\
    &\leq \expectation_{\action_0, \state_1} \expectation_{\action_1, \state_2, \ldots | \state_1} \Big[ |\tfunc(\state_0) - Y - \tfunc(\state_0) + \expectation_{\action_1, \state_2, \ldots | \state_1} [Y] |\Big] \label{eq:jensen_gap_proof_g_1_lip}\\
    &= \expectation_{\action_0, \state_1} \expectation_{\action_1, \state_2, \ldots | \state_1} \Big[ |\expectation_{\action_1, \state_2, \ldots | \state_1} [Y] - Y |\Big] \\
    &\leq \expectation_{\action_0, \state_1} \sqrt{\expectation_{\action_1, \state_2, \ldots | \state_1} \Big[ (\expectation_{\action_1, \state_2, \ldots | \state_1} [Y] - Y )^2\Big] } \label{eq:jensen_gap_proof_g_cs}\\
    &= \expectation_{\action_0, \state_1} \sigma_{Y|\state_1} \\
    &= \sigma_Y
\end{align}
Eq.~\ref{eq:jensen_gap_proof_g_1_lip} is due to the fact that $\max\{0,a\}$ is 1-Lipschitz, and we use Cauchy-Schwarz inequality in Eq.~\ref{eq:jensen_gap_proof_g_cs}. $\sigma_{Y|\state_1}$ is the standard deviation of the random variable $Y = \min_{\tidx \in \{1, \dots, \tidxgoal\}} \tfunc(\state_\tidx)$ starting from state $\state_1$, and its expectation over action $\action_0$ and $\state_1$ is denoted by $\sigma_Y$. Hence, $G \leq \sigma_Y$.

From the 3 cases, we conclude that $|\vfuncpi(\state)- \expectation_{\action_0 \sim \policy} \expectation_{\state_1 \sim \tdyn} \min \left\{ \tfunc(\state_0), \vfuncpi(\state_1) \right\}| \leq \sigma_Y$. 
\end{proof}

\section{Proof of Proposition~\ref{prop:discounted_liveness_update_contraction}}\label{appx:contraction_proof}
\begin{proof}
Let $B$ denote the discounted Bellman operator defined in Eq.~\ref{eq:discounted_liveness_update}. We would like to show that there exists $\kappa\in[0,1)$ such that $\left| B[\vfuncpiover](\state) - B[\vfuncpiover'](\state) \right| \leq \kappa ||\vfuncpiover - \vfuncpiover'||_\infty$. Take $\state\in\sspace$, and let us denote the immediate next state of $\state$ with $\state'$. Also take $\vfuncpiover$ and $\vfuncpiover'$. In Eq.~\ref{eq:contraction_proof_min_line}, we make use of the identity $|\min\{a,b\} - \min\{a,c\}| \leq |b-c|$. The full derivation is as follows. 

\begin{align}
\bigg| &B[\vfuncpiover](\state) - B[\vfuncpiover'](\state) \bigg| \nonumber \\
    &= \bigg| \cancel{(1-\dfac)} + \dfac \min \{ \tfunc(\state), \vfuncpiover(\state') \} \nonumber \\
    &\qquad - \left( \cancel{(1-\dfac)} + \dfac \min \{ \tfunc(\state), \vfuncpiover'(\state') \} \right) \bigg| \\
    &= \dfac \left| \min \{ \tfunc(\state), \vfuncpiover(\state') \} - \min \{ \tfunc(\state), \vfuncpiover'(\state') \} \right| \label{eq:contraction_proof_min_line}\\
    &\le \dfac \left| \vfuncpiover(\state') - \vfuncpiover'(\state') \right|
\end{align}
Since $\big| B[\vfuncpiover](\state) - B[\vfuncpiover'](\state) \big| \leq \dfac \left| \vfuncpiover(\state') - \vfuncpiover'(\state') \right|$ holds $\forall \state\in\sspace$, we have 
\begin{equation}
    \big| B[\vfuncpiover](\state) - B[\vfuncpiover'](\state) \big| \leq \dfac||\vfuncpiover - \vfuncpiover'||_\infty
\end{equation}
Since the discount factor $\dfac$ is chosen in the open interval of $(0,1)$, it serves as the contraction constant $\kappa$ we look for in this proof. Hence, we have shown that for $\kappa = \dfac \in (0,1)$, $\left| B[\vfuncpiover](\state) - B[\vfuncpiover'](\state) \right| \leq \kappa ||\vfuncpiover - \vfuncpiover'||_\infty$.
\end{proof}

\newpage
\section{Method Hyperparameters}\label{appx:case_study_hyperparameters}

The default hyperparameters for our experiments for the LIBERO-Spatial task, Square task, and Cloth Folding task are provided in TABLE~\ref{tab:libero-hyperparams}, \ref{tab:square-hyperparams}, and \ref{tab:cloth-folding-hyperparams}, respectively. 

\begin{table}[h]
\centering
\caption{LIBERO-Spatial Task Hyperparameters}
\label{tab:libero-hyperparams}
\begin{tabularx}{.8\columnwidth}{@{}Xl@{}}
\toprule
\textbf{Hyperparameter} & \textbf{Value} \\ \midrule
\textit{Architecture} & \\
MLP Layers & 5 \\
Hidden Units & 512 \\
Activation & GELU \\
Normalization & Layer Norm \\
Input Dimension & 1544 (SigLIP2) \\ \midrule
\textit{Training} & \\
Learning Rate & $1 \times 10^{-5}$ \\
Batch Size & 512 \\
Training Epochs & 5,000 \\ \midrule
\textit{Task \& Method Constants} & \\
Standard Task Horizon ($T$) & 200 \\
Our Method Discount factor ($\gamma$) & 0.993 \\
TD(0) Discount factor ($\gamma$) & 0.999 \\
TD(0) Reward Normalization & 250 \\
MC / MC-D Reward Normalization & 500 \\
MC-D Bins & 201 \\ \midrule
\textit{Prioritized Experience Replay} & \\
Replay Buffer Capacity & 10,000 \\
Gradient Steps per Batch & 2 \\
$\alpha$ & 0.6 \\
$\beta_{start}$ & 0.4 \\
$\beta_{end}$ & 1 \\ \bottomrule
\end{tabularx}
\end{table}

\begin{table}[h]
\centering
\caption{Square Task Hyperparameters}
\label{tab:square-hyperparams}
\begin{tabularx}{.8\columnwidth}{@{}Xl@{}}
\toprule
\textbf{Hyperparameter} & \textbf{Value} \\ \midrule
\textit{Architecture} & \\
MLP Layers & 5 \\
Hidden Units & 512 \\
Activation & GELU \\
Normalization & Layer Norm \\
Input Dimension & 3090 (SigLIP2) \\ \midrule
\textit{Training} & \\
Learning Rate & $1 \times 10^{-5}$ \\
Batch Size & 512 \\
Training Epochs & 5,000 \\ \midrule
\textit{Task \& Method Constants} & \\
Standard Task Horizon ($T$) & 100 \\
Our Method Discount factor ($\gamma$) & 0.993 \\
TD(0) Discount factor ($\gamma$) & 0.995 \\
TD(0) Reward Normalization & 200 \\
MC / MC-D Reward Normalization & 400 \\
MC-D Bins & 201 \\ \midrule
\textit{Prioritized Experience Replay} & \\
Replay Buffer Capacity & 10,000 \\
Gradient Steps per Batch & 2 \\
$\alpha$ & 0.6 \\
$\beta_{start}$ & 0.4 \\
$\beta_{end}$ & 1 \\ \bottomrule
\end{tabularx}
\end{table}

\begin{table}[h]
\centering
\caption{Cloth Folding Task Hyperparameters}
\label{tab:cloth-folding-hyperparams}
\begin{tabularx}{.8\columnwidth}{@{}Xl@{}}
\toprule
\textbf{Hyperparameter} & \textbf{Value} \\ \midrule
\textit{Architecture} & \\
MLP Layers & 5 \\
Hidden Units & 512 \\
Activation & GELU \\
Normalization & Layer Norm \\
Input Dimension & 768 (SigLIP2) \\ \midrule
\textit{Training} & \\
Learning Rate & $2 \times 10^{-5}$ \\
Batch Size & 512 \\
Training Epochs & 5,000 \\ \midrule
\textit{Task \& Method Constants} & \\
Standard Task Horizon ($T$) & 100 \\
Our Method Discount factor ($\gamma$) & 0.99 \\
TD(0) Discount factor ($\gamma$) & 0.99 \\
TD(0) Reward Normalization & 200 \\
MC / MC-D Reward Normalization & 400 \\
MC-D Bins & 201 \\ \midrule
\textit{Prioritized Experience Replay} & \\
Replay Buffer Capacity & 10,000 \\
Gradient Steps per Batch & 2 \\
$\alpha$ & 0.6 \\
$\beta_{start}$ & 0.4 \\
$\beta_{end}$ & 1 \\ \bottomrule
\end{tabularx}
\end{table}

\clearpage

\section{Detailed Experiment Results}\label{appx:experiment_tables}
In this appendix, we present detailed experiment results for the three case studies. Each experiment was run with 5 random seeds and the result is reported in the format \texttt{mean}$\pm$\texttt{standard deviation}. In TABLE~\ref{tab:metrics_results_full_dataset}, we present results for all 3 metrics for the experiments with full datasets (200 episodes for the simulation tasks and 150 episodes for the hardware cloth folding task). Furthermore, we present the results for the dataset ablation, described in Sec.~\ref{subsec:dataset_size_ablation}, in TABLE~\ref{tab:dataset_ablation_success_metric}, \ref{tab:dataset_ablation_failure_metric}, and \ref{tab:dataset_ablation_composite_metric}.

% Table 1: largest dataset size, all metrics
\begin{table}[ht]
\centering
\footnotesize
\begin{tabular}{llccc}
\toprule
Task & Method & Success Metric ($\uparrow$) & Failure Metric ($\uparrow$) & Composite Metric ($\uparrow$) \\
\midrule
 & \shade Ours & \shade $\mathbf{0.9796 \pm 0.0039}$ & \shade $0.9242 \pm 0.0037$ & \shade $\mathbf{0.9519 \pm 0.0017}$ \\
 & \shade Ours-NB & \shade $0.8397 \pm 0.0079$ & \shade $0.9767 \pm 0.0055$ & \shade $0.9082 \pm 0.0043$ \\
 & TD(0) & $0.7667 \pm 0.0101$ & $0.9974 \pm 0.0005$ & $0.8821 \pm 0.0048$ \\
 & MC & $0.7012 \pm 0.0117$ & $\mathbf{0.9993 \pm 0.0003}$ & $0.8502 \pm 0.0057$ \\
\multirow{-5}{*}{LIBERO-Spatial} & MC-D & $0.6706 \pm 0.0296$ & $0.9715 \pm 0.0036$ & $0.8210 \pm 0.0154$ \\
\midrule
 & \shade Ours & \shade $\mathbf{0.7278 \pm 0.0242}$ & \shade $0.8914 \pm 0.0233$ & \shade $\mathbf{0.8096 \pm 0.0017}$ \\
 & \shade Ours-NB & \shade $0.6283 \pm 0.0083$ & \shade $0.9217 \pm 0.0057$ & \shade $0.7750 \pm 0.0069$ \\
 & TD(0) & $0.5937 \pm 0.0081$ & $0.9267 \pm 0.0024$ & $0.7602 \pm 0.0032$ \\
 & MC & $0.5634 \pm 0.0101$ & $\mathbf{0.9337 \pm 0.0021}$ & $0.7486 \pm 0.0047$ \\
\multirow{-5}{*}{Square} & MC-D & $0.5214 \pm 0.0072$ & $0.9180 \pm 0.0015$ & $0.7197 \pm 0.0041$ \\
\midrule
 & \shade Ours & \shade $\mathbf{0.6429 \pm 0.0304}$ & \shade $0.9227 \pm 0.0039$ & \shade $\mathbf{0.7828 \pm 0.0164}$ \\
 & \shade Ours-NB & \shade $0.3038 \pm 0.0120$ & \shade $0.9713 \pm 0.0047$ & \shade $0.6376 \pm 0.0064$ \\
 & TD(0) & $0.2392 \pm 0.0082$ & $0.0641 \pm 0.0090$ & $0.1516 \pm 0.0036$ \\
 & MC & $0.1995 \pm 0.0106$ & $\mathbf{0.9911 \pm 0.0019}$ & $0.5953 \pm 0.0059$ \\
\multirow{-5}{*}{Cloth Folding} & MC-D & $0.2444 \pm 0.0052$ & $0.9536 \pm 0.0039$ & $0.5990 \pm 0.0039$ \\
\midrule\bottomrule
\end{tabular}
\caption{Results of all 3 metrics for full dataset for all 3 case studies across 5 training seeds.}
\label{tab:metrics_results_full_dataset}
\end{table}

% Table 2: success metric across dataset sizes
\begin{table}[ht]
\centering
\footnotesize
\begin{tabular}{ll c c c c}
\toprule
& & \multicolumn{4}{c}{Success Metric ($\uparrow$)} \\
\cmidrule(lr){3-6}
Task & Method & 50 & 100 & 150 & 200 \\
\midrule
 & \shade Ours & \shade $\mathbf{0.8383 \pm 0.0096}$ & \shade $\mathbf{0.9172 \pm 0.0072}$ & \shade $\mathbf{0.9469 \pm 0.0048}$ & \shade $\mathbf{0.9796 \pm 0.0039}$ \\
 & \shade Ours-NB & \shade $0.6571 \pm 0.0094$ & \shade $0.7399 \pm 0.0024$ & \shade $0.8020 \pm 0.0054$ & \shade $0.8397 \pm 0.0079$ \\
 & TD(0) & $0.4826 \pm 0.0130$ & $0.6220 \pm 0.0114$ & $0.6945 \pm 0.0108$ & $0.7667 \pm 0.0101$ \\
 & MC & $0.3887 \pm 0.0178$ & $0.5277 \pm 0.0095$ & $0.6243 \pm 0.0068$ & $0.7012 \pm 0.0117$ \\
\multirow{-5}{*}{LIBERO-Spatial} & MC-D & $0.4088 \pm 0.0224$ & $0.5356 \pm 0.0295$ & $0.6163 \pm 0.0539$ & $0.6706 \pm 0.0296$ \\
\midrule
 & \shade Ours & \shade $\mathbf{0.6981 \pm 0.0085}$ & \shade $\mathbf{0.7247 \pm 0.0101}$ & \shade $\mathbf{0.7072 \pm 0.0033}$ & \shade $\mathbf{0.7278 \pm 0.0242}$ \\
 & \shade Ours-NB & \shade $0.6261 \pm 0.0087$ & \shade $0.6467 \pm 0.0081$ & \shade $0.6025 \pm 0.0445$ & \shade $0.6283 \pm 0.0083$ \\
 & TD(0) & $0.4886 \pm 0.0093$ & $0.5813 \pm 0.0079$ & $0.5837 \pm 0.0106$ & $0.5937 \pm 0.0081$ \\
 & MC & $0.4349 \pm 0.0059$ & $0.5603 \pm 0.0081$ & $0.5521 \pm 0.0079$ & $0.5634 \pm 0.0101$ \\
\multirow{-5}{*}{Square} & MC-D & $0.4566 \pm 0.0121$ & $0.5171 \pm 0.0143$ & $0.5047 \pm 0.0170$ & $0.5214 \pm 0.0072$ \\
\midrule
 & \shade Ours & \shade $\mathbf{0.4443 \pm 0.0110}$ & \shade $\mathbf{0.5589 \pm 0.0207}$ & \shade $\mathbf{0.6429 \pm 0.0304}$ & \shade -- \\
 & \shade Ours-NB & \shade $0.2895 \pm 0.0062$ & \shade $0.2894 \pm 0.0081$ & \shade $0.3038 \pm 0.0120$ & \shade -- \\
 & TD(0) & $0.2784 \pm 0.0053$ & $0.2351 \pm 0.0079$ & $0.2392 \pm 0.0082$ & -- \\
 & MC & $0.2394 \pm 0.0042$ & $0.2060 \pm 0.0083$ & $0.1995 \pm 0.0106$ & -- \\
\multirow{-5}{*}{Cloth Folding} & MC-D & $0.1855 \pm 0.0170$ & $0.1975 \pm 0.0240$ & $0.2444 \pm 0.0052$ & -- \\
\midrule\bottomrule
\end{tabular}
\caption{Success Metric across the dataset size ablation.}
\label{tab:dataset_ablation_success_metric}
\end{table}

% Table 3: failure metric across dataset sizes
\begin{table}[ht]
\centering
\footnotesize
\begin{tabular}{ll c c c c}
\toprule
& & \multicolumn{4}{c}{Failure Metric ($\uparrow$)} \\
\cmidrule(lr){3-6}
Task & Method & 50 & 100 & 150 & 200 \\
\midrule
 & \shade Ours & \shade $0.9256 \pm 0.0034$ & \shade $0.9491 \pm 0.0029$ & \shade $0.9369 \pm 0.0088$ & \shade $0.9242 \pm 0.0037$ \\
 & \shade Ours-NB & \shade $0.9862 \pm 0.0078$ & \shade $0.9785 \pm 0.0180$ & \shade $0.9775 \pm 0.0057$ & \shade $0.9767 \pm 0.0055$ \\
 & TD(0) & $0.9966 \pm 0.0005$ & $0.9987 \pm 0.0008$ & $0.9995 \pm 0.0008$ & $0.9974 \pm 0.0005$ \\
 & MC & $\mathbf{0.9980 \pm 0.0004}$ & $\mathbf{0.9994 \pm 0.0002}$ & $\mathbf{0.9997 \pm 0.0004}$ & $\mathbf{0.9993 \pm 0.0003}$ \\
\multirow{-5}{*}{LIBERO-Spatial} & MC-D & $0.9850 \pm 0.0045$ & $0.9828 \pm 0.0035$ & $0.9797 \pm 0.0027$ & $0.9715 \pm 0.0036$ \\
\midrule
 & \shade Ours & \shade $0.8688 \pm 0.0013$ & \shade $0.8906 \pm 0.0058$ & \shade $0.9054 \pm 0.0026$ & \shade $0.8914 \pm 0.0233$ \\
 & \shade Ours-NB & \shade $0.8818 \pm 0.0049$ & \shade $0.9020 \pm 0.0072$ & \shade $0.9190 \pm 0.0039$ & \shade $0.9217 \pm 0.0057$ \\
 & TD(0) & $0.8471 \pm 0.0908$ & $0.9164 \pm 0.0074$ & $0.9134 \pm 0.0355$ & $0.9267 \pm 0.0024$ \\
 & MC & $\mathbf{0.9268 \pm 0.0020}$ & $\mathbf{0.9230 \pm 0.0030}$ & $\mathbf{0.9360 \pm 0.0012}$ & $\mathbf{0.9337 \pm 0.0021}$ \\
\multirow{-5}{*}{Square} & MC-D & $0.9093 \pm 0.0045$ & $0.9132 \pm 0.0018$ & $0.9246 \pm 0.0052$ & $0.9180 \pm 0.0015$ \\
\midrule
 & \shade Ours & \shade $0.8295 \pm 0.0027$ & \shade $0.9152 \pm 0.0048$ & \shade $0.9227 \pm 0.0039$ & \shade -- \\
 & \shade Ours-NB & \shade $0.8967 \pm 0.0117$ & \shade $0.9339 \pm 0.0616$ & \shade $0.9713 \pm 0.0047$ & \shade -- \\
 & TD(0) & $0.1646 \pm 0.0125$ & $0.0703 \pm 0.0041$ & $0.0641 \pm 0.0090$ & -- \\
 & MC & $\mathbf{0.9537 \pm 0.0061}$ & $\mathbf{0.9929 \pm 0.0028}$ & $\mathbf{0.9911 \pm 0.0019}$ & -- \\
\multirow{-5}{*}{Cloth Folding} & MC-D & $0.9478 \pm 0.0107$ & $0.9665 \pm 0.0085$ & $0.9536 \pm 0.0039$ & -- \\
\midrule\bottomrule
\end{tabular}
\caption{Failure Metric across the dataset size ablation.}
\label{tab:dataset_ablation_failure_metric}
\end{table}

% Table 4: composite metric across dataset sizes
\begin{table}[ht]
\centering
\footnotesize
\begin{tabular}{ll c c c c}
\toprule
& & \multicolumn{4}{c}{Composite Metric ($\uparrow$)} \\
\cmidrule(lr){3-6}
Task & Method & 50 & 100 & 150 & 200 \\
\midrule
 & \shade Ours & \shade $\mathbf{0.8820 \pm 0.0052}$ & \shade $\mathbf{0.9332 \pm 0.0027}$ & \shade $\mathbf{0.9419 \pm 0.0045}$ & \shade $\mathbf{0.9519 \pm 0.0017}$ \\
 & \shade Ours-NB & \shade $0.8216 \pm 0.0058$ & \shade $0.8592 \pm 0.0096$ & \shade $0.8897 \pm 0.0034$ & \shade $0.9082 \pm 0.0043$ \\
 & TD(0) & $0.7396 \pm 0.0065$ & $0.8104 \pm 0.0055$ & $0.8470 \pm 0.0054$ & $0.8821 \pm 0.0048$ \\
 & MC & $0.6933 \pm 0.0087$ & $0.7635 \pm 0.0047$ & $0.8120 \pm 0.0034$ & $0.8502 \pm 0.0057$ \\
\multirow{-5}{*}{LIBERO-Spatial} & MC-D & $0.6969 \pm 0.0099$ & $0.7592 \pm 0.0143$ & $0.7980 \pm 0.0271$ & $0.8210 \pm 0.0154$ \\
\midrule
 & \shade Ours & \shade $\mathbf{0.7834 \pm 0.0037}$ & \shade $\mathbf{0.8076 \pm 0.0045}$ & \shade $\mathbf{0.8063 \pm 0.0015}$ & \shade $\mathbf{0.8096 \pm 0.0017}$ \\
 & \shade Ours-NB & \shade $0.7540 \pm 0.0049$ & \shade $0.7744 \pm 0.0050$ & \shade $0.7608 \pm 0.0211$ & \shade $0.7750 \pm 0.0069$ \\
 & TD(0) & $0.6678 \pm 0.0420$ & $0.7488 \pm 0.0052$ & $0.7485 \pm 0.0213$ & $0.7602 \pm 0.0032$ \\
 & MC & $0.6808 \pm 0.0032$ & $0.7417 \pm 0.0030$ & $0.7441 \pm 0.0037$ & $0.7486 \pm 0.0047$ \\
\multirow{-5}{*}{Square} & MC-D & $0.6829 \pm 0.0066$ & $0.7151 \pm 0.0072$ & $0.7146 \pm 0.0069$ & $0.7197 \pm 0.0041$ \\
\midrule
 & \shade Ours & \shade $\mathbf{0.6369 \pm 0.0050}$ & \shade $\mathbf{0.7371 \pm 0.0120}$ & \shade $\mathbf{0.7828 \pm 0.0164}$ & \shade -- \\
 & \shade Ours-NB & \shade $0.5931 \pm 0.0040$ & \shade $0.6117 \pm 0.0288$ & \shade $0.6376 \pm 0.0064$ & \shade -- \\
 & TD(0) & $0.2215 \pm 0.0085$ & $0.1527 \pm 0.0047$ & $0.1516 \pm 0.0036$ & -- \\
 & MC & $0.5965 \pm 0.0027$ & $0.5995 \pm 0.0053$ & $0.5953 \pm 0.0059$ & -- \\
\multirow{-5}{*}{Cloth Folding} & MC-D & $0.5667 \pm 0.0090$ & $0.5820 \pm 0.0159$ & $0.5990 \pm 0.0039$ & -- \\
\midrule\bottomrule
\end{tabular}
\caption{Composite Metric across the dataset size ablation.}
\label{tab:dataset_ablation_composite_metric}
\end{table}

% Table 5: cloth folding encoder ablation (Our Method P2, size 150)
\begin{table}[ht]
\centering
\footnotesize
\begin{tabular}{lccc}
\toprule
Encoder & Success Metric ($\uparrow$) & Failure Metric ($\uparrow$) & Composite Metric ($\uparrow$) \\
\midrule
SigLIP2 & $0.6429 \pm 0.0304$ & $\mathbf{0.9227 \pm 0.0039}$ & $0.7828 \pm 0.0164$ \\
CLIP & $0.6793 \pm 0.0096$ & $0.9105 \pm 0.0053$ & $0.7949 \pm 0.0063$ \\
DINOv2 & $\mathbf{0.6878 \pm 0.0250}$ & $0.9203 \pm 0.0089$ & $\mathbf{0.8040 \pm 0.0124}$ \\
\bottomrule
\end{tabular}
\caption{Cloth Folding task encoder ablation at dataset size 150 across training seeds.}
\label{tab:encoder_ablation_results}
\end{table}

\clearpage

\end{document}